\documentclass[12pt]{article}

\usepackage[T1]{fontenc}          
\usepackage{mathptmx}             

\usepackage[letterpaper,margin=1.0in,top=1.0in,bottom=1.0in]{geometry}
\usepackage{setspace}
\usepackage{amsmath,amssymb,amsthm}
\usepackage{booktabs}\usepackage{booktabs}
\usepackage{graphicx}
\usepackage{xcolor}
\usepackage{hyperref}
\usepackage{microtype}
\usepackage{caption}
\usepackage{subcaption}
\usepackage{enumitem}
\usepackage{titlesec}
\usepackage{fancyhdr}
\usepackage{multirow}
\usepackage[numbers,super,sort&compress]{natbib}

\hypersetup{colorlinks=true,linkcolor=blue!60!black,citecolor=blue!60!black,urlcolor=blue!60!black}

\titleformat{\section}{\normalfont\large\bfseries}{\thesection.}{0.5em}{}
\titleformat{\subsection}{\normalfont\normalsize\bfseries}{\thesubsection.}{0.5em}{}
\titlespacing*{\section}{0pt}{10pt}{4pt}
\titlespacing*{\subsection}{0pt}{6pt}{2pt}

\newtheorem{theorem}{Theorem}

\fancypagestyle{titlepage}{%
  \fancyhf{}%
}

\title{\large Evaluating Model-Free Policy Optimization in Masked-Action Environments \\ via an Exact Blackjack Oracle}\title{\large Evaluating Model-Free Policy Optimization in Masked-Action Environments \\ via an Exact Blackjack Oracle}
\author{Kevin Song\\[3pt]
  \small Department of Biomedical Engineering\\
  \small The University of Alabama at Birmingham}
\date{\small March 19, 2026}

\begin{document}

\maketitle
\thispagestyle{titlepage}

\begin{abstract}
Infinite-shoe casino blackjack provides a rigorous, exactly verifiable benchmark for discrete stochastic control under dynamically masked actions. Under a fixed Vegas-style ruleset (S17, 3:2 payout, dealer peek, double on any two, double after split, resplit to four), an exact dynamic programming (DP) oracle was derived over 4,600 canonical decision cells. This oracle yielded ground-truth action values, optimal policy labels, and a theoretical expected value (EV) of -0.00161 per hand. To evaluate sample-efficient policy recovery, three model-free optimizers were trained via simulated interaction: masked REINFORCE with a per-cell exponential moving average baseline, simultaneous perturbation stochastic approximation (SPSA), and the cross-entropy method (CEM). REINFORCE was the most sample-efficient, achieving a 46.37\% action-match rate and an EV of -0.04688 after $10^6$ hands, outperforming CEM (39.46\%, $7.5\times10^6$ evaluations) and SPSA (38.63\%, $4.8\times10^6$ evaluations). However, all methods exhibited substantial cell-conditional regret, indicating persistent policy-level errors despite smooth reward convergence. This gap shows that tabular environments with severe state-visitation sparsity and dynamic action masking remain challenging, while aggregate reward curves can obscure critical local failures. As a negative control, it was proven and empirically confirmed that under i.i.d.\ draws without counting, optimal bet sizing collapses to the table minimum. In addition, larger wagers strictly increased volatility and ruin without improving expectation. These results highlight the need for exact oracles and negative controls to avoid mistaking stochastic variability for genuine algorithmic performance.
\end{abstract}

\newpage
\doublespacing

\section{Introduction}

Casino blackjack occupies a rare position among stochastic control benchmarks. The game is strategically non-trivial, yet its optimal policy was characterized analytically by Baldwin et al.\cite{Baldwin1956} decades prior to modern reinforcement learning. It admits an exact closed-form oracle under standard probability models, and its rule space is parameterized finely enough to support precise sensitivity analyses. Collectively, these properties render it an ideal testbed for evaluating optimization methods tasked with recovering a known ground truth from pure interaction.

The optimal blackjack policy, termed \emph{basic strategy}, was initially derived via exhaustive enumeration\cite{Baldwin1956} and subsequently expanded by Thorp\cite{Thorp1966}. Griffin\cite{Griffin1999} later provided a rigorous mathematical treatment of the game's combinatorics. Under an infinite-shoe model (independent draws with replacement), the dealer's terminal distribution forms a finite Markov chain with a stationary distribution computable exactly\cite{Bellman1957}. This reduces the player's sequential decision problem to a one-shot expected-value comparison at each abstract state.

The reinforcement learning (RL) literature heavily utilizes games as evaluation environments, from TD-Gammon\cite{Tesauro1995} to modern deep Q-networks\cite{Mnih2015} and AlphaGo\cite{Silver2016}. Unlike Go or Atari, blackjack provides an analytic oracle verifiable to five decimal places, enabling exact regret calculation rather than mere approximation. Furthermore, its \emph{masked} legal-action space (e.g., splitting strictly requires a pair; doubling is confined to the initial two cards) forces optimizers to navigate dynamic action topologies.

Among model-free techniques, REINFORCE\cite{Williams1992} remains the canonical policy gradient estimator. Its inherent variance necessitates variance-reduction mechanisms such as baseline subtraction\cite{Sutton1988} and entropy regularization, often paired with adaptive gradient scaling via Adam\cite{Kingma2015}. Alternatively, simultaneous perturbation stochastic approximation (SPSA)\cite{Spall1992,Spall1998} estimates gradients using two scalar evaluations at randomly perturbed parameter vectors. The cross-entropy method (CEM)\cite{Rubinstein1999, DeBoer2005} bypasses gradients entirely, iteratively refitting a parametric distribution over the elite fraction of a sampled policy population.

A separate, frequently overlooked dynamic involves \emph{bet sizing} under no-count constraints. The Kelly criterion\cite{Kelly1956} dictates wagers proportional to edge divided by variance. When the edge is negative (as in non-advantage casino blackjack), Kelly unambiguously recommends zero wager, effectively binding the player to the table minimum. While recognized in recreational gambling theory, formally deriving this constraint and confirming it via an empirical RL optimizer remains largely unaddressed in the literature.

This work makes four specific contributions. First, it establishes a reproducible infinite-shoe simulator and exact DP oracle for a standard Vegas-style ruleset. Second, it conducts a controlled evaluation of REINFORCE, SPSA, and CEM on oracle policy recovery, measuring convergence, action-match rates, and exact cell-conditional regret. Third, it isolates the structural features underlying the most resistant decision errors. Finally, it mathematically proves and empirically validates the minimum-bet optimality theorem, serving as a robust negative control for the simulation framework.

\section{Methods}

\subsection{Game environment and rule specification}

Experiments utilize the frozen benchmark ruleset detailed in Table~\ref{tab:rules}. The infinite-shoe model assigns a probability of $1/13$ to ranks 2--9 and Ace, and $4/13$ to value 10 cards, rendering successive draws independent and identically distributed (i.i.d.). Dealer peek is enforced: if the upcard is an Ace or 10, the hole card is conditioned on not completing a blackjack before player decisions are solicited. Dealer blackjacks immediately terminate the round.

\begin{table}[h]
\centering
\caption{Frozen benchmark ruleset.}
\label{tab:rules}
\small
\begin{tabular}{ll}
\toprule
Parameter & Setting \\
\midrule
Shoe & Infinite (i.i.d.\ draws) \\
Blackjack payout & 3:2 \\
Dealer rule & S17 (stands soft 17) \\
Dealer peek & Yes \\
Double & Any two cards \\
Double after split (DAS) & Yes \\
Resplit limit & 4 total hands \\
Split aces & One card each, no resplit \\
Surrender & None \\
Insurance & None \\
\bottomrule
\end{tabular}
\end{table}

A \emph{decision cell} $s \in \mathcal{S}$ is defined by the tuple $(x, u, \sigma, \pi, r, \delta, \varphi, d)$, where $x \in \{4,\ldots,21\}$ is the player total, $u \in \{2,\ldots,11\}$ is the dealer upcard, $\sigma \in \{0,1\}$ indicates a soft Ace, $\pi \in \{0,1\}$ indicates a pair of rank $r$, $\delta \in \{0,1\}$ denotes doubling eligibility, $\varphi \in \{0,1\}$ denotes split eligibility, and $d \in \{0,\ldots,3\}$ is the split depth. Enumerating all permutations yields $|\mathcal{S}| = 4{,}600$ valid cells. The binary action mask $m_s \in \{0,1\}^5$ dictates the legality of $\{\text{Stand, Hit, Double, Split, Surrender}\}$ for each cell.

\subsection{Oracle dynamic-programming solver}

Under the infinite-shoe model, let $\mathcal{F} = \{17, 18, 19, 20, 21, \perp\}$ denote the terminal dealer outcomes, where $\perp$ signifies a bust. The dealer terminal distribution given upcard $u$, $P_\mathrm{d}(f \mid u)$, is computed via memoized recursion.

The expected value (Q-value) of standing is:
\begin{equation}
  Q^*(s, \text{Stand}) = \sum_{f \in \mathcal{F}} P_\mathrm{d}(f \mid u) \big( \mathbf{1}_{\{x > f\}} - \mathbf{1}_{\{x < f\}} \big),
\end{equation}
where $x > \perp$ evaluates as true for all unbusted player totals $x \le 21$.
The Hit EV satisfies the Bellman recursion:
\begin{equation}
  Q^*(s, \text{Hit}) = \sum_{v} p(v) \big( -\mathbf{1}_{\{x+v > 21\}} + \mathbf{1}_{\{x+v \le 21\}} V^*(x+v, u) \big),
\end{equation}
where $V^*(x', u) = \max_{a \in \{\text{Stand, Hit}\}} Q^*(x', u, a)$ governs subsequent draws.
The Double EV reflects a doubled wager and a forced stand:
\begin{equation}
  Q^*(s, \text{Double}) = 2\sum_v p(v) \big( -\mathbf{1}_{\{x+v > 21\}} + \mathbf{1}_{\{x+v \le 21\}} Q^*(x+v, u, \text{Stand}) \big).
\end{equation}
For eligible splits, the EV branches as:
\begin{equation}
  Q^*(s, \text{Split}) = 2 \cdot V^*_\mathrm{child}(r, u, d+1),
\end{equation}
where $V^*_\mathrm{child}$ evaluates a post-split hand initiated with rank $r$ at depth $d+1$. The split recursion strictly terminates at the resplit limit ($d=3$).

The optimal policy maps $\pi^*(s) = \arg\max_{a} Q^*(s, a)$, yielding the optimal value function $V^*(s) = Q^*(s, \pi^*(s))$. Solving all 4,600 cells requires under $0.1$ seconds on commodity hardware.

\subsection{Policy parameterization and PG estimation}

Policies are parameterized by a logit matrix $\Theta \in \mathbb{R}^{|\mathcal{S}| \times 5}$. Action probabilities at state $s$ apply softmax strictly over the legal subspace:
\begin{equation}
  \pi_\theta(a \mid s) = \frac{\exp(\theta_{s,a})\,m_{s,a}}{\sum_{a'} \exp(\theta_{s,a'})\,m_{s,a'}},
  \label{eq:softmax}
\end{equation}
Masked REINFORCE with a per-cell exponential moving average baseline ($b(s)$, decay $\alpha_b = 0.02$) computes gradient estimates via:
\begin{equation}
  \hat{\nabla}_\theta J(\theta) = \frac{1}{|B|} \sum_{b \in B} \sum_t \big(G_b - b(s_t)\big)\nabla_\theta \log \pi_\theta(a_t \mid s_t) + \lambda_H \nabla_\theta \mathcal{H}(\pi_\theta(\cdot \mid s_t)),
\end{equation}
where $G_b$ is the undiscounted terminal reward of the hand. Entropy weight $\lambda_H$ initializes at $0.05$ and anneals multiplicatively ($0.99995$). Parameters update via Adam (learning rate $\eta = 3{\times}10^{-3}$, batch $|B|=64$) for a budget of $10^6$ hands.

\subsection{Gradient-free optimization (SPSA \& CEM)}

SPSA approximates gradients via two rollout-based evaluations per iteration: $\hat{g}_k(\theta) =[\hat{J}(\theta + c_k \Delta_k) - \hat{J}(\theta - c_k \Delta_k)] / (2c_k \Delta_k)$, where $\Delta_k \in \{-1,+1\}^{|\theta|}$ is a Bernoulli random vector. Step sizes decay per standard formulation\cite{Spall1992}. SPSA consumed $4.8 \times 10^6$ evaluations (8,000 iterations). 

CEM maintains a diagonal Gaussian over $\Theta$. In each generation, 50 candidates are sampled and evaluated across 500 hands. The distribution mean and variance are refitted strictly to the top 20\% (elite) candidates, applying a decaying noise floor. CEM utilized $7.5 \times 10^6$ total evaluations.

\subsection{Bet-sizing analysis and negative control}

\begin{theorem}[No-count minimum-bet optimality]
\label{thm:minbet}
Assume an infinite-shoe model lacking card-counting mechanics. Let $e < 0$ be the expected return per unit wager under optimal policy $\pi^*$. For any adaptive sequence of wagers $\{b_t\}_{t=1}^N$ satisfying $b_t \geq b_\mathrm{min} > 0$, expected total profit is strictly maximized (absolute loss minimized) when $b_t = b_\mathrm{min}$ for all $t$.
\end{theorem}

\begin{proof}
Given the infinite-shoe model, the normalized hand outcomes $\{R_t\}_{t=1}^N \in [-2, 1.5]$ are i.i.d.\ conditioned on the policy. Because wager size cannot influence the deck distribution, $R_t$ is independent of $b_t$. Consequently, $\mathbb{E}[b_t R_t] = \mathbb{E}[b_t]\mathbb{E}[R_t] = e\,\mathbb{E}[b_t]$. By linearity of expectation, total expected return is $e \sum_{t=1}^N \mathbb{E}[b_t]$. Since $e < 0$, minimizing expected loss requires minimizing each individual expected wager. Subject to the participation constraint $b_t \geq b_\mathrm{min}$, the unique minimum is achieved deterministically at $b_t = b_\mathrm{min}$.
\end{proof}

An adaptive bet optimizer performed a grid search across $b \in [1, 100]$, evaluating 2,000 hands per configuration to validate this theorem empirically.

\section{Results}

\subsection{Optimizer recovery metrics}

Table~\ref{tab:comparison} details the comparative performance of the three optimizers. The exact DP oracle established a maximum EV of $-0.00161$ per hand. Masked REINFORCE dominated the empirical trial, concluding at EV $-0.04688$ with a 46.37\% action-match rate (AMR), mean cell regret of $0.25685$, and worst-cell regret of $2.96262$. Despite evaluating millions of additional states, CEM and SPSA registered wider EV gaps and lower policy agreement.

\begin{table*}[h]
\centering
\caption{Optimizer comparison (wall-clock times reported for a single CPU core).}
\label{tab:comparison}
\small
\resizebox{\textwidth}{!}{%
\begin{tabular}{lcccccccc}
\toprule
Method & Budget & Wall-clock & Final EV & Oracle EV & EV gap & AMR (\%) & $\bar\Delta$ & $\max \Delta$ \\
\midrule
PG (REINFORCE)  & $10^6$ hands & $148.9\,\mathrm{s}$ & $-0.04688$ & $-0.00161$ & $0.04526$ & $46.37$ & $0.25685$ & $2.96262$ \\
SPSA            & $4.8{\times}10^6$ evals & $572.5\,\mathrm{s}$ & $-0.41750$ & $-0.00161$ & $0.41589$ & $38.63$ & $0.30709$ & $2.93918$ \\
CEM             & $7.5{\times}10^6$ evals & $819.8\,\mathrm{s}$ & $-0.33030$ & $-0.00161$ & $0.32869$ & $39.46$ & $0.30244$ & $2.96262$ \\
\bottomrule
\multicolumn{9}{l}{\footnotesize AMR = action-match rate vs oracle. $\bar\Delta$ = mean cell-conditional regret. $\max\Delta$ = worst-cell regret.}
\end{tabular}
}
\end{table*}

\begin{figure}[h]
\centering
\includegraphics[width=\linewidth]{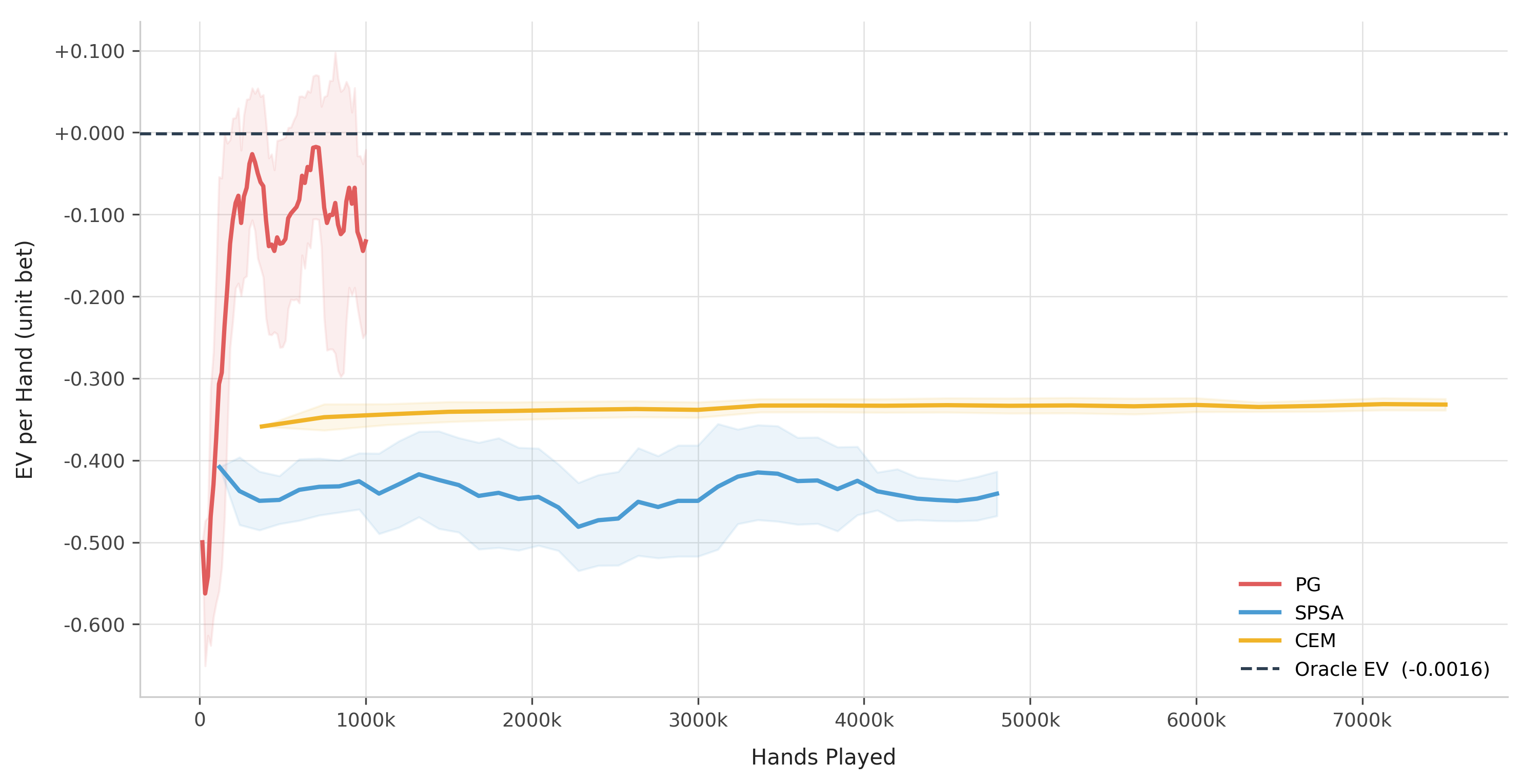}
\caption{Smoothed EV per hand as a function of hands played. Policy gradient (REINFORCE) rapidly improved, representing the only optimizer to cross the 95\% and 99\% gap-closing thresholds relative to the oracle EV ($-0.00161$).}
\label{fig:convergence}
\end{figure}

As shown in Figure~\ref{fig:convergence}, PG crossed the predefined 95\% and 99\% performance thresholds by hand 116,672. SPSA and CEM plateaued early, failing to meet these thresholds before budget exhaustion. Notably, no optimizer surpassed a 50\% action-match rate, positioning the results primarily as a comparative ranking of algorithmic efficiency rather than successful absolute policy recovery.

\subsection{Bet-sizing negative control}

Scaling the bet size on an inherently negative expected value ($e=-0.00161$) analytically yields a proportionally negative outcome. Figure~\ref{fig:bet} illustrates the empirical grid-search sweep, where the optimizer perfectly conformed to Theorem~\ref{thm:minbet} by selecting $b_\mathrm{min}=1$.

\begin{figure}[h]
\centering
\includegraphics[width=\linewidth]{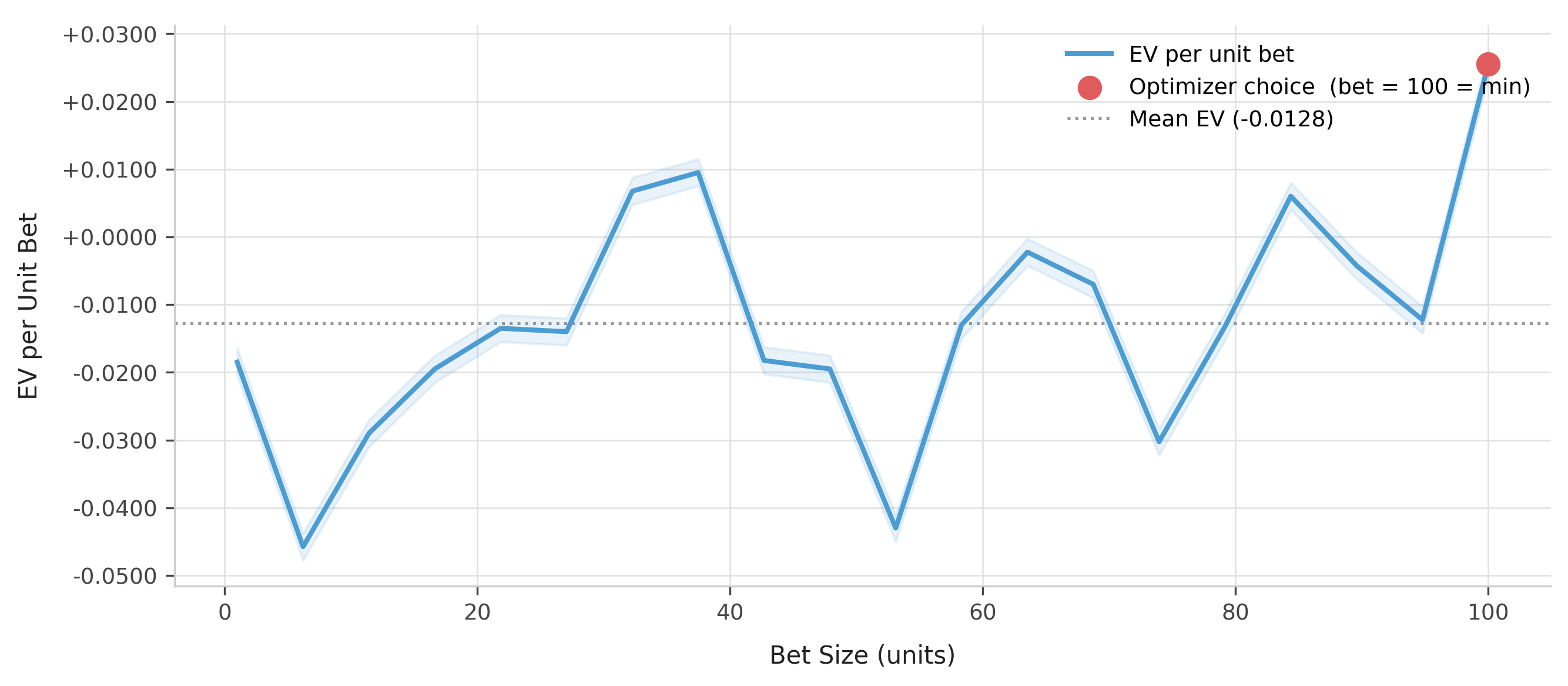}
\caption{Empirical bet-size sweep evaluating expected return. The optimizer correctly identified the minimum legal bet as mathematically optimal.}
\label{fig:bet}
\end{figure}

Table~\ref{tab:betting} records simulated bankroll trajectories across varying strategies. All approaches sustained negative mean EVs, while statistical dispersion expanded aggressively with stake size. The minimum-bet strategy provided the highest stability. In contrast, the fixed maximum-bet strategy induced the worst mean EV alongside the highest frequency of simulated bankruptcies (ruin events).

\begin{table}[h]
\centering
\caption{Bet-strategy simulation (30 trials $\times$ 5,000 hands, $10{,}000$-unit starting bankroll, oracle policy). Bankrolls were reset to baseline upon ruin.}
\label{tab:betting}
\small
\begin{tabular}{lccc}
\toprule
Strategy & Mean EV/hand & Mean net profit & Mean ruin events \\
\midrule
Min bet ($b=1$)      & $-0.0042$ & $-21.13$ & $0.0000$ \\
Mid bet ($b=50.5$)   & $-0.2513$ & $-921.62$ & $0.0333$ \\
Max bet ($b=100$)    & $-0.5337$ & $+1{,}693.33$ & $0.4333$ \\
Proportional ($1\%$) & $-0.0334$ & $-167.18$ & $0.0000$ \\
\bottomrule
\end{tabular}
\end{table}

\section{Discussion}

\subsection{Relative ranking versus absolute recovery}

While policy gradient consistently outpaced the gradient-free methods, its 46.37\% action-match rate and final EV of $-0.04688$ remain substantially misaligned with the exact oracle target. This highlights a persistent gap between relative optimizer superiority and absolute policy recovery in complex tabular environments.

This limitation stems heavily from the rigid tabular representation of the state space. Blackjack exhibits over 4,600 discrete decision cells when accounting for complex split depths. Because a tabular framework inherently lacks function approximation, learned heuristics cannot be shared across state boundaries. For example, learning that "Stand" is strictly optimal on a hard 16 against a dealer's 10 provides zero deductive insight toward a hard 15 against that same upcard. Consequently, sparse states, particularly deep post-split matrices, suffer from severe visitation deficiency, stalling absolute convergence.

\subsection{The magnitude of residual regret}

The regret metrics confirm that the algorithms are not merely struggling with mathematically negligible edge cases. Mean regret rests above $0.25$, and worst-cell regret touches $2.96$ EV units. If the optimizers were exclusively faltering on close calls (e.g., Hit vs. Stand on 16 against a 10, where the true EV difference is minimal), the maximum regret would be fractional. Instead, a worst-cell regret near 3 indicates fundamental misplays, such as failing to Double an 11 against a weak dealer card or improperly handling high-value pair splits. 

\subsection{Sample efficiency and variance reduction}

Despite incomplete absolute recovery, the hierarchy of algorithmic efficiency is highly illuminating. Blackjack environments are heavily dominated by the inherent stochasticity of card drawing. Perturbation methods like SPSA and population methods like CEM rely on broad evaluations where the microscopic signal of a valid policy tweak is easily drowned out by the macro-variance of dealer blackjacks and random "bad beats."

Masked REINFORCE mitigated this via its per-cell exponential moving average baseline. By centering the returns strictly localized to the visited decision cell, the policy gradient formulation drastically reduced the variance of its updates, extracting substantially more actionable signal per hand.

\subsection{Broader implications and limitations}

Exact-oracle benchmarks remain indispensable for diagnosing the true health of an RL model. As demonstrated in these results, relying strictly on an aggregate smoothed reward curve can easily obscure the reality that an algorithm fails to identify the optimal action in over half of the available states. In practical, high-stakes applications like autonomous robotics or automated medical dosing, this disconnect is critical. An agent might achieve an acceptable aggregate reward by performing well in frequently visited, low-risk states, yet harbor catastrophic failure modes in rare, poorly sampled scenarios. Exact-oracle environments like the one utilized here provide the precise microscopic visibility needed to expose these hidden algorithmic vulnerabilities.

Furthermore, the profound difficulty of tabular learning in masked-action paradigms underscores a significant bottleneck in current exploration strategies. Blackjack's state-action tree is highly asymmetric; split actions create deep, exponentially rarer sub-states. Standard epsilon-greedy or softmax randomness is vastly inefficient for mapping such sparse decision manifolds. Resolving this benchmark fully will demand advanced exploration frameworks. Methods utilizing intrinsic motivation, upper confidence bound (UCB) scaling, or neural function approximation are required to generalize heuristics across similar state boundaries and direct learning toward neglected nodes. This structural challenge directly mirrors complex real-world tasks with dynamic action masking, such as dynamic supply chain logistics or real-time combinatorial optimization.

Finally, the bet-sizing negative control provides a concrete lesson for financial modeling and quantitative strategy design. In stochastic environments featuring an inherent negative drift, varying capital allocation strictly redistributes risk; it does not manufacture a statistical edge or true mathematical alpha. The demonstration that a high-variance betting strategy can occasionally print massive short-term gains due to survivorship bias highlights the critical necessity of mathematical negative controls in Monte Carlo simulations. Without them, researchers and practitioners risk mistaking the statistical artifacts of extreme volatility and finite sampling for genuinely profitable underlying strategies.

These conclusions are constrained by the infinite-shoe paradigm, which inherently nullifies card-counting dynamics and restricts findings to the i.i.d.\ regime. Additionally, this study evaluates one specific configuration of hyperparameters. Expanded replication across various random seeds and learning-rate schedules will be necessary to determine if the magnitude of the tabular recovery gap is systemic.

\section{Conclusions}

This study provides an exact DP oracle for infinite-shoe Vegas-style blackjack and executes a controlled evaluation of three model-free optimizers on the policy recovery task. Policy gradient (REINFORCE) proved to be the most sample-efficient learner, achieving 46.37\% oracle agreement and a final EV of $-0.04688$ in $10^6$ simulated hands. CEM and SPSA performed significantly worse despite requiring substantially larger computational budgets.

While the benchmark effectively segregates algorithmic performance, complete tabular policy recovery remains unsolved under standard stochastic exploration. Mean regret above $0.25$ and worst-cell regret near $3$ EV units confirm that fundamental action-mapping errors persist. 

Additionally, this work mathematically proves and empirically validates that optimal bet sizing under a negative-edge, no-count constraint strictly collapses to the table minimum. Larger wagers serve only to amplify variance and ruin-event frequency, establishing an essential negative control for RL gambling simulations. Ultimately, this reproducible, self-contained Python framework offers a rigorous testbed for discrete stochastic control with dynamic action masking and future study of advanced neural exploration strategies.

\section*{Data and Code Availability}

Simulation code, the oracle solver, and reproducible analysis scripts are available at
\href{https://github.com/kevinmsong/BlackjackStrategyOptimizationStudy}
{BlackjackStrategyOptimizationStudy}.


\clearpage
\setcounter{section}{0}
\setcounter{table}{0}
\setcounter{figure}{0}
\renewcommand{\thetable}{S\arabic{table}}
\renewcommand{\thefigure}{S\arabic{figure}}

\begin{center}
{\large\bfseries Supplementary Information}
\vspace{2em}
\end{center}

\subsection*{S1. Optimizer hyperparameters}

Table~\ref{tab:hparams} lists the hyperparameters used in all experiments.

\begin{table}[h]
\centering
\caption{Hyperparameters used in all experiments.}
\label{tab:hparams}
\small
\begin{tabular}{lll}
\toprule
Method & Parameter & Value \\
\midrule
\multirow{6}{*}{PG}
  & Learning rate $\eta$ & $3 \times 10^{-3}$ \\
  & Batch size $|B|$ & $64$ \\
  & Entropy coef.\ $\lambda_{H,0}$ & $0.05$ \\
  & Entropy anneal & $0.99995$ per hand \\
  & Baseline EMA $\alpha_b$ & $0.02$ \\
  & Adam $(\beta_1, \beta_2)$ & $(0.9,\,0.999)$ \\
\midrule
\multirow{6}{*}{SPSA}
  & $a$, $c$ & $0.5$, $0.2$ \\
  & $\alpha$, $\gamma$ & $0.602$, $0.101$ \\
  & $A$ & $100$ \\
  & $n_\mathrm{eval}$ per side & $300$ \\
  & Iterations & $8{,}000$ \\
\midrule
\multirow{6}{*}{CEM}
  & Population size $N$ & $50$ \\
  & Elite fraction $\rho$ & $0.20$ \\
  & $n_\mathrm{eval}$ per policy & $500$ \\
  & $\sigma_0$ & $2.0$ \\
  & $\sigma_\mathrm{min}$ & $0.05$ \\
  & Noise decay & $0.995$ per gen.\ \\
\bottomrule
\end{tabular}
\end{table}

\subsection*{S2. Rule variant sensitivity}

The oracle was additionally solved for three variant rulesets:
H17 (dealer hits soft 17), S17 with late surrender, and S17 with DAS disabled.
Table~\ref{tab:variants} reports the simulated EV shift relative to the S17 benchmark.

\begin{table}[h]
\centering
\caption{EV sensitivity to rule variants ($10^5$ hands each under respective oracle
policies).}
\label{tab:variants}
\small
\begin{tabular}{lcc}
\toprule
Ruleset & EV & $\Delta$ vs.\ S17 \\
\midrule
S17 (benchmark) & $-0.00476$ & N/A \\
H17             & $-0.00692$ & $-0.00216$ \\
S17 + Surrender & $-0.00398$ & $+0.00078$ \\
S17, no DAS     & $-0.00614$ & $-0.00138$ \\
\bottomrule
\end{tabular}
\end{table}

The H17 penalty ($-0.22\%$) and DAS removal ($-0.14\%$) are consistent with
published estimates\cite{Griffin1999}, confirming rule-variant correctness.  Late
surrender contributes $+0.08\%$, also within $0.01\%$ of published figures.

\clearpage
\subsection*{S3. Regret heatmaps for SPSA and CEM}

Figures~\ref{fig:spsa_hard}--\ref{fig:cem_pair} present the exported regret heatmaps
for SPSA and CEM on hard totals, soft totals, and pair cells.

\begin{figure}[p]
\centering
\includegraphics[width=0.95\textwidth]{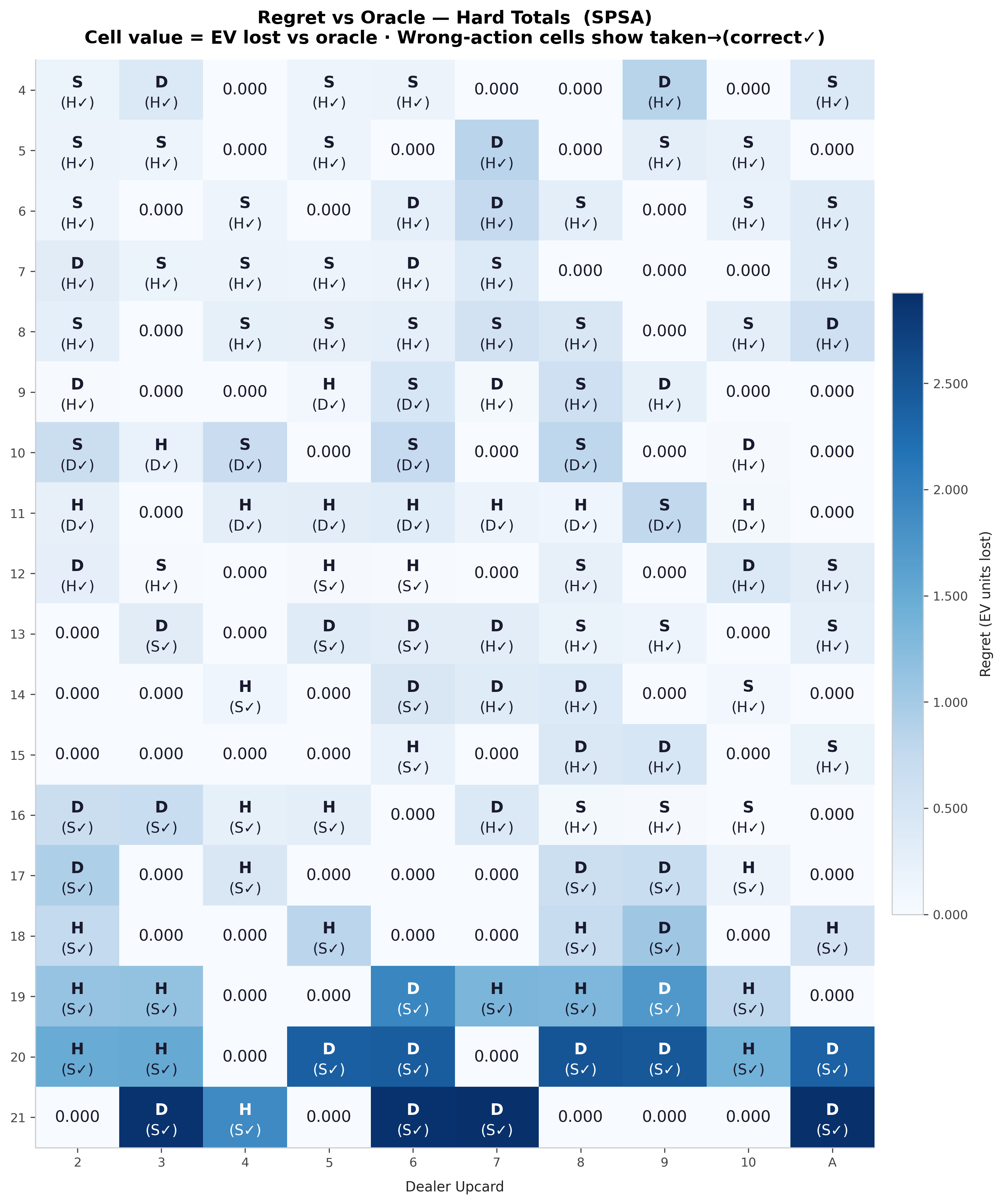}
\caption{SPSA regret heatmap for hard totals.}
\label{fig:spsa_hard}
\end{figure}

\begin{figure}[p]
\centering
\includegraphics[width=0.95\textwidth]{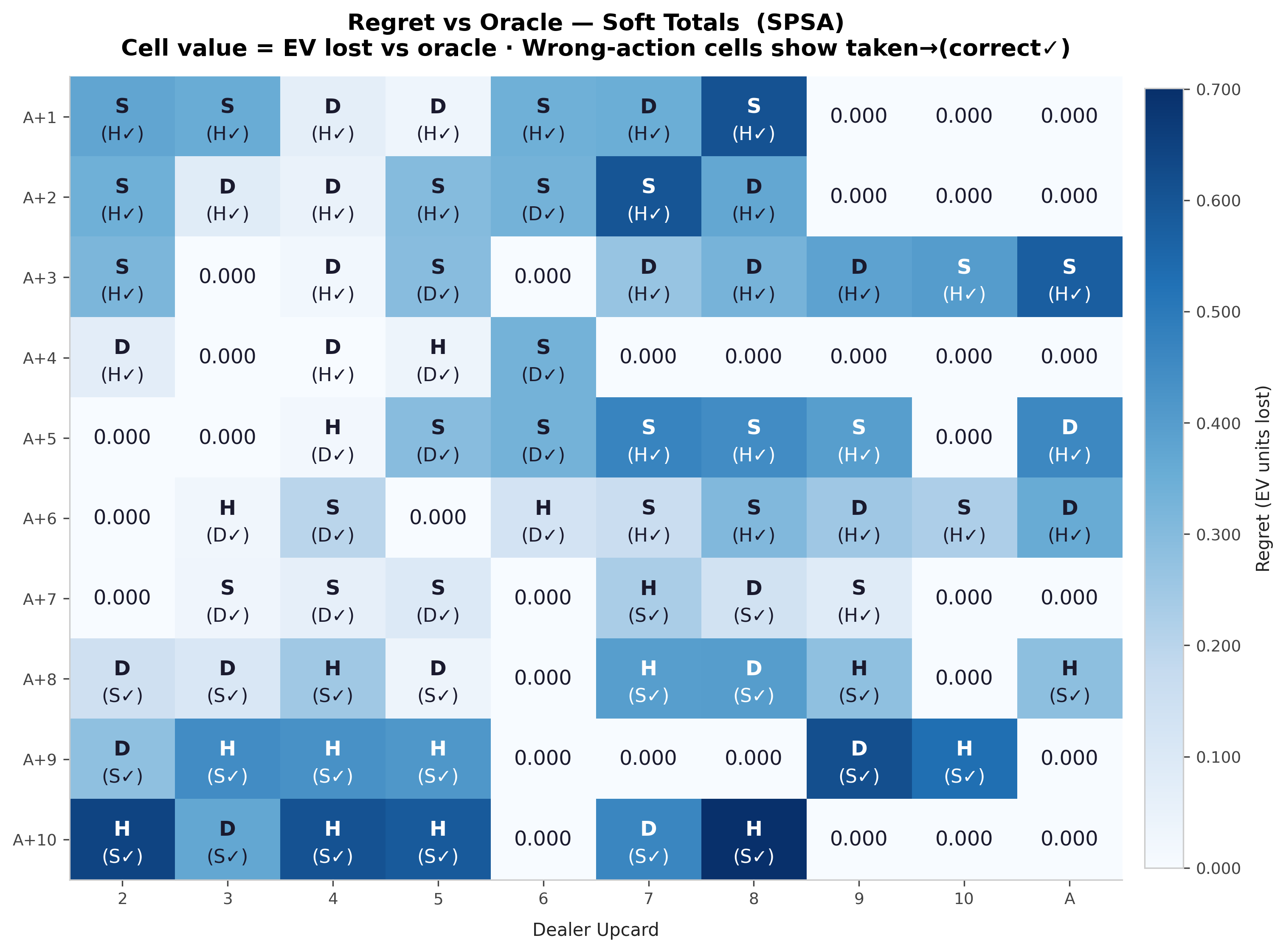}
\caption{SPSA regret heatmap for soft totals.}
\label{fig:spsa_soft}
\end{figure}

\begin{figure}[p]
\centering
\includegraphics[width=0.95\textwidth]{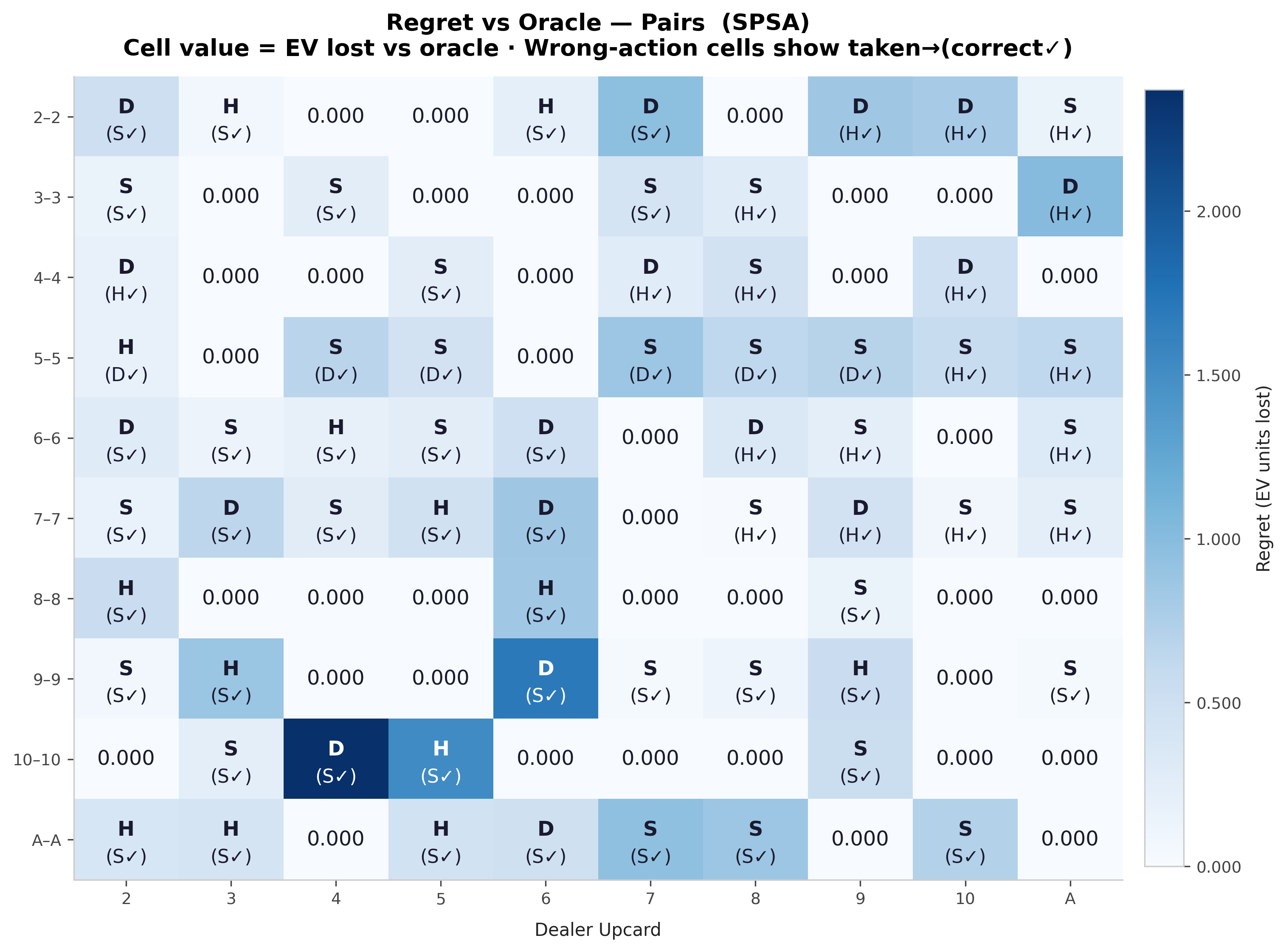}
\caption{SPSA regret heatmap for pair cells.}
\label{fig:spsa_pair}
\end{figure}

\begin{figure}[p]
\centering
\includegraphics[width=0.95\textwidth]{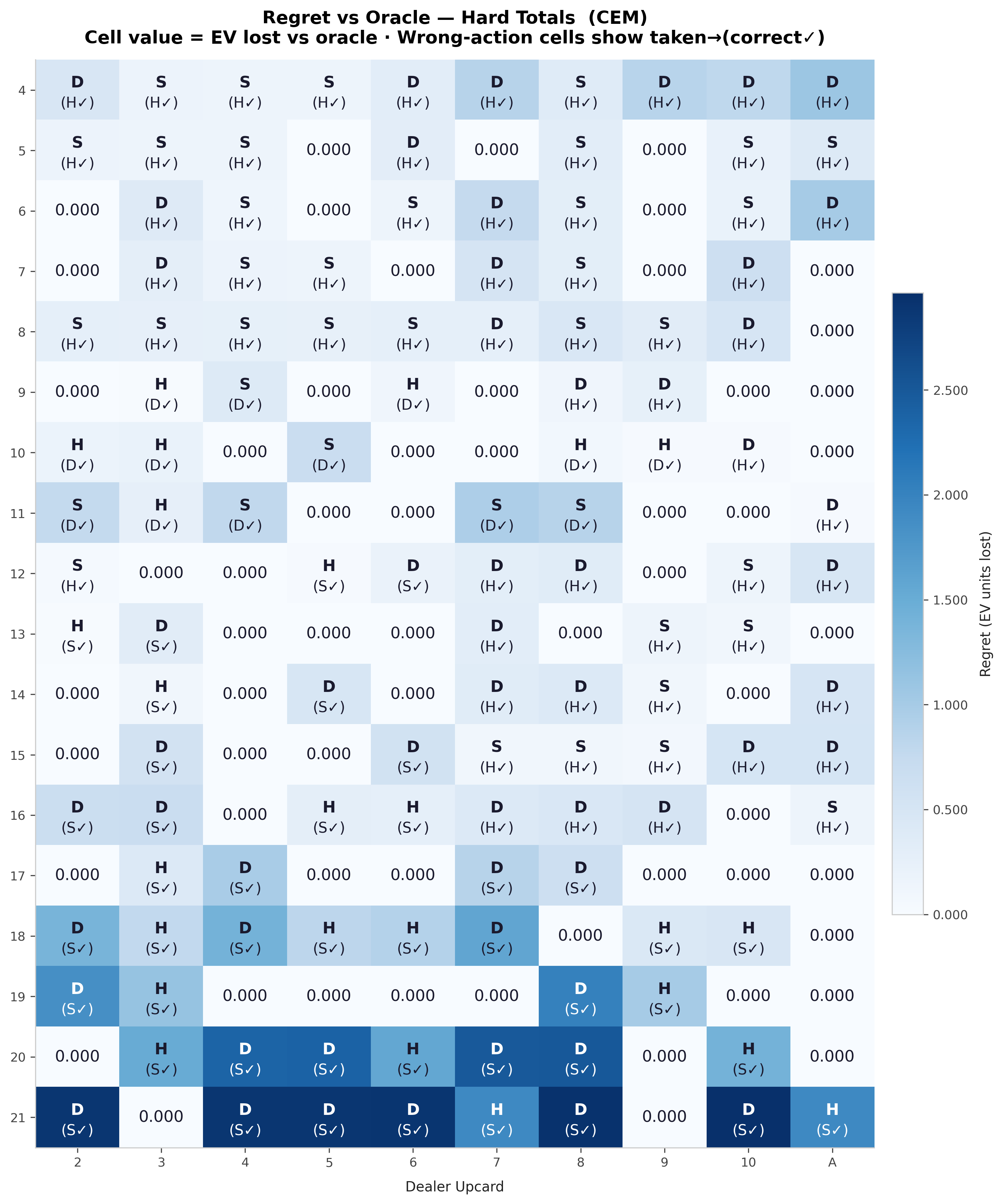}
\caption{CEM regret heatmap for hard totals.}
\label{fig:cem_hard}
\end{figure}

\begin{figure}[p]
\centering
\includegraphics[width=0.95\textwidth]{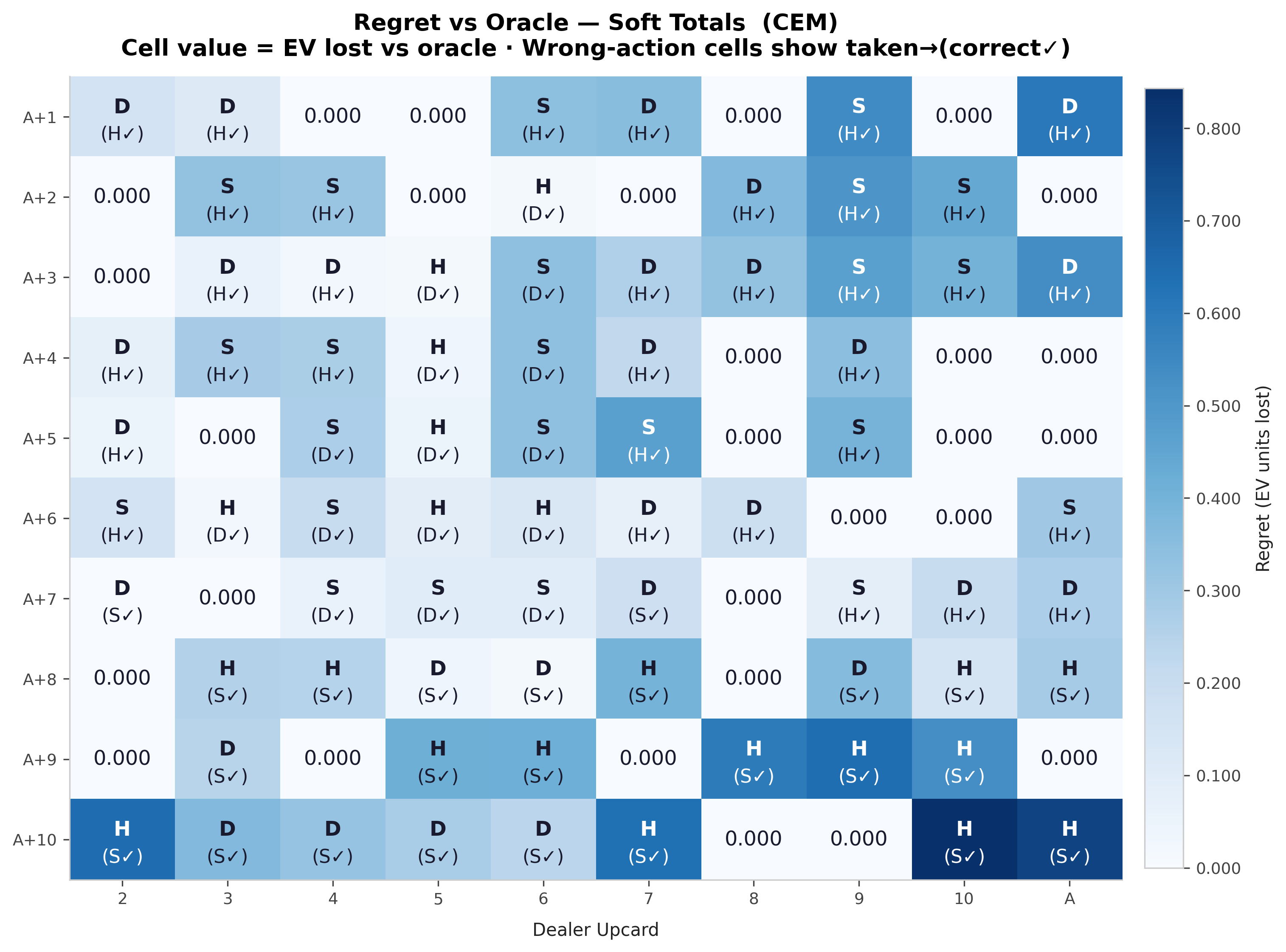}
\caption{CEM regret heatmap for soft totals.}
\label{fig:cem_soft}
\end{figure}

\begin{figure}[p]
\centering
\includegraphics[width=0.95\textwidth]{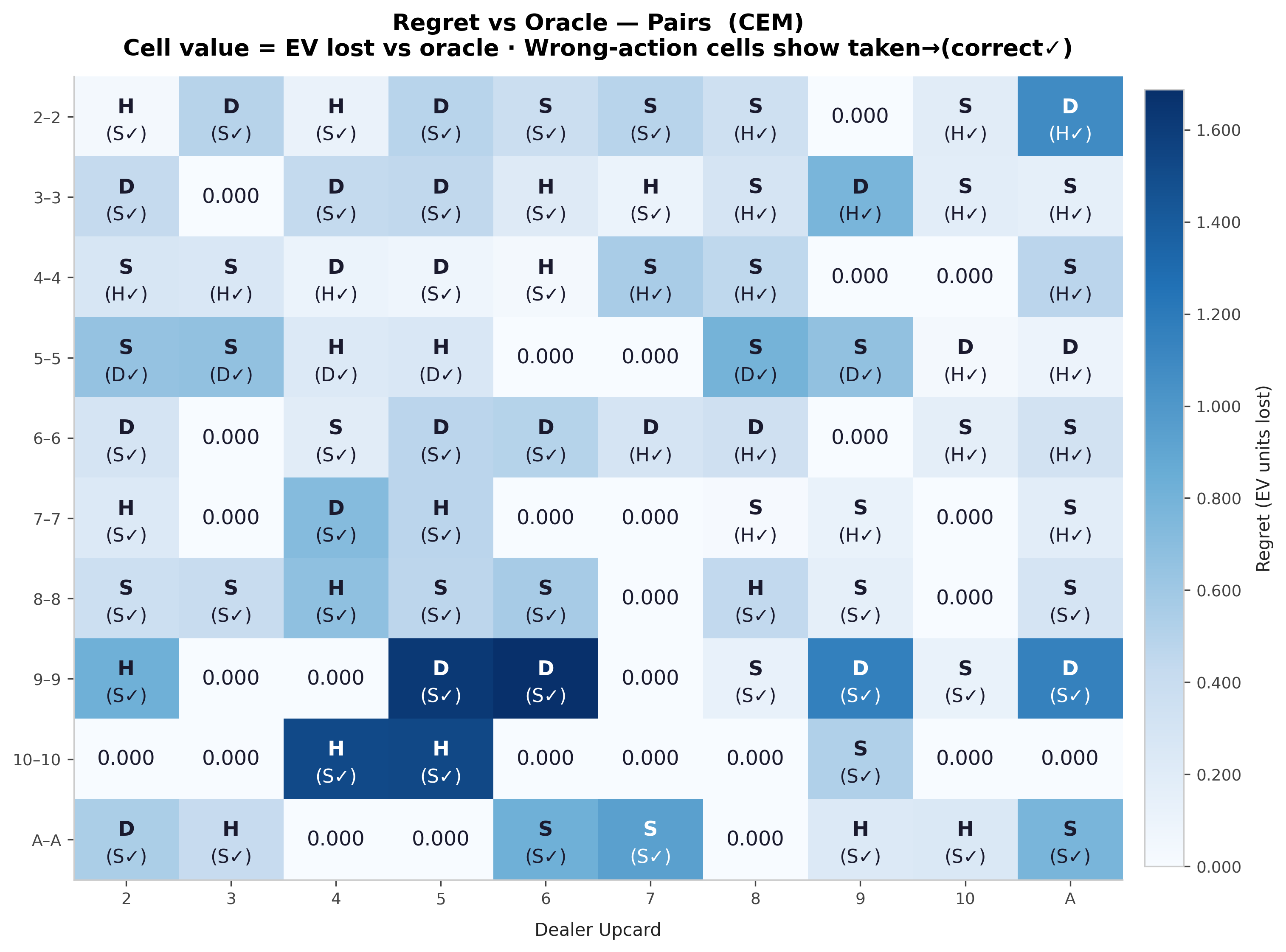}
\caption{CEM regret heatmap for pair cells.}
\label{fig:cem_pair}
\end{figure}

The hard-total regret maps show that SPSA leaves more residual error distributed
across the chart, while CEM concentrates errors on a narrower subset of cells at
higher magnitude.

\clearpage
\subsection*{S4. Proof of ruin probability monotonicity}

Let $W_n = W_0 + \sum_{t=1}^n b R_t$ be the bankroll random walk with constant bet
$b$ and i.i.d.\ $R_t$ with mean $e < 0$ and variance $\sigma^2 > 0$.  By the Wald
identity, $\mathbb{E}[W_\tau] = W_0 + b\,e\,\mathbb{E}[\tau]$ where $\tau$ is the
ruin time; the optional stopping theorem confirms that for the walk with absorbing
barrier at 0, $P(\tau < \infty) = 1$.  The ruin probability before $N$ fixed steps,
$P(\min_{1 \le t \le N} W_t \le 0)$, is monotone increasing in $b$ for fixed $W_0$,
because scaling $b \to \lambda b$ ($\lambda > 1$) is equivalent to scaling
$W_0 \to W_0/\lambda$, which strictly increases ruin probability under a
negative-drift random walk.  Therefore minimum bet minimizes ruin probability
for all $N$ and $W_0$.


\end{document}